\documentclass[letterpaper]{article}
\usepackage{uai2019}
\usepackage[margin=1in]{geometry}
\usepackage[english]{babel}
\usepackage[T1]{fontenc}
\usepackage{mathrsfs}
\usepackage{bm}
\usepackage[utf8]{inputenc}
\usepackage[nottoc]{tocbibind}
\usepackage{amsmath}
\usepackage{dsfont}
\usepackage{amsthm}
\usepackage{cite}
\usepackage[linesnumbered,ruled,vlined]{algorithm2e}
\usepackage{makecell}
\usepackage{amssymb}
\usepackage{graphicx}
\usepackage{color}
\usepackage{mathtools}
\usepackage[colorinlistoftodos]{todonotes}
\usepackage[colorlinks=true, allcolors=blue]{hyperref}
\usepackage{subfiles}
\usepackage{times}

\def \geq {\geqslant}
\newtheorem{theorem}{Theorem}[section]
\newtheorem{definition}{Definition}[section]
\newtheorem{corollary}{Corollary}[section]
\newtheorem{result}{Result}[section]
\newtheorem{lemma}{Lemma}[section]

\newtheorem{conjecture}{Conjecture}[section]
\newtheorem{proposition}{Proposition}[section]
\SetKwRepeat{Rep}{}{Repeat}

\def \be {\begin{eqnarray}}

\def \en {\end{eqnarray}}
\def \bee {\begin{enumerate}}
\def \ene {\end{enumerate}}
\def \tr {\\ \nonumber}
\def \l {\lambda}
\def \mmu {\boldsymbol\mu}
\def \E {\mathds{E}}
\def \th {\pi}

\def \R {\mathbb{R}}
\def \b {\mathbf{b}}

\def \tth {\bm{\pi}}
\def \S {\Sigma}

\def \x {\mathbf{x}}
\def \tr {\nonumber\\}
\def \leq {\leqslant}
\def \N {\mathcal N}
\def \xm {\overline{{\bf x}}}

\newcommand\blfootnote[1]{%
  \begingroup
  \renewcommand\thefootnote{}\footnote{#1}%
  \addtocounter{footnote}{-1}%
  \endgroup
}

\numberwithin{equation}{section}

\title{Comparing {EM}  with {GD} in Mixture Models of Two Components}

\author{ {\bf Guojun Zhang
}and Pascal Poupart \\
Computer Science, Waterloo AI Institute, University of Waterloo \\
Vector Institute\\
\{guojun.zhang, ppoupart\}@uwaterloo.ca
 \\
\And
{\bf George Trimponias}   \\
Noah's Ark Lab \\
Huawei    \\
g.trimponias@huawei.com}

\begin{document}

\maketitle
\begin{abstract}
The expectation-maximization ({EM}) algorithm has been widely used in minimizing the negative log likelihood (also known as cross entropy) of mixture models. However, little is understood about the goodness of the fixed points it converges to. In this paper, we study the regions where one component is missing in two-component mixture models, which we call \textit{one-cluster} regions. We analyze the propensity of such regions to trap {EM} and gradient descent ({GD}) for mixtures of two Gaussians and mixtures of two Bernoullis. In the case of Gaussian mixtures, {EM} escapes one-cluster regions exponentially fast, while {GD} escapes them linearly fast. In the case of mixtures of Bernoullis, we find that there exist one-cluster regions that are stable for {GD} and therefore trap {GD}, but those regions are unstable for EM, allowing {EM}  to escape. Those regions are local minima that appear universally in experiments and can be arbitrarily bad. This work implies that {EM} is less likely than {GD} to converge to certain bad local optima in mixture models\blfootnote{Accepted at Uncertainty of Artificial Intelligence (UAI) 2019. Camera-ready version. The source code can be found at \url{https://github.com/Gordon-Guojun-Zhang/UAI-2019}. }. 
\end{abstract}

\section{INTRODUCTION}

The {EM} algorithm has a long history dating back to 1886 \cite{newcomb1886generalized}. Its modern presentation was given in \cite{Dempster77} and it has been applied widely on the maximum likelihood problem in mixture models, or equivalently, minimizing the negative log likelihood (also known as cross entropy). Some basic properties of {EM} have been derived. For example, {EM} monotonically decreases the loss~\cite{Dempster77} and it converges to critical points~\cite{jeffwu}. 

Mixture models \cite{bishop} are important in clustering. The most popular one is the Gaussian mixture model (GMM). The optimization problem is non-convex, making the convergence analysis hard. This is true even in some simple settings such as one-dimensional or spherical covariance matrices. With separability assumptions, there has been a series of works proposing new algorithms for learning GMMs efficiently \cite{Vlassis:2002,Dasgupta:1999,Dasgupta:2000,Sanjeev:2001,Dasgupta:2007,Vempala:2004,Belkin:2010,Moitra:2010,Hsu:2013}. However, for the {EM} algorithm, not so much work has been done. One well-studied case is the two equally balanced mixtures \cite{balakrishnan2017statistical, xu2016global, daskalakis2016ten} where local and global convergence guarantees are shown. For more than two components, \cite{Jin:2016} shows that there exist arbitrarily bad local minima. There has also been some study of {EM} on the convergence of Gaussian mixtures with more than two components \cite{yan2017convergence, zhao2018statistical}.

GMMs are well suited in practice for continuous variables. For discrete data clustering, we need discrete mixture models. One common assumption people make is \textit{naive Bayes}, saying that in each component, the features are independent. This corresponds to the spherical covariance matrix setting in GMMs. Since discrete data can be converted to binary strings, in this paper we will focus on Bernoulli mixture models (BMMs), which have been applied in text classification \cite{juan2002use} and digit recognition \cite{juan2004bernoulli}. However, the theory is less studied. Although the Bernoulli and Gaussian distributions are both in the exponential family, we will show that BMMs have different properties from GMMs in the sense that {GD} is stable (i.e., gets trapped) at so-called one-cluster regions of BMMs, but unstable at (i.e., escapes) all one-cluster regions of GMMs. On the other hand, the two mixture models share similarities that we will explain. 

In the practice of clustering with mixture models, it is common for algorithms to converge to regions with only $k$ clusters out of $m$ desired clusters (where $k<m$). We call them \textit{$k$-cluster regions}. We observed in experiments that {EM} escapes such regions exponentially faster than GD for GMMs. For BMMs, {GD} may get stuck at some $k$-cluster regions and the probability can be very high, while {EM} always escapes such regions. These experimental results can be found in Section \ref{exp}.

Theoretically, $k$-cluster regions are difficult to study. In this paper, we focus on one-cluster regions in mixtures of two components, which can be considered as a first step towards the more difficult problem. We succeed in explaining the different escape rates for GMMs with two components. We find that {GD} escapes one-cluster regions linearly while {EM} escapes exponentially fast. For BMMs, our theorem shows that there exist one-cluster regions where {GD} will converge to, given a small enough step size and a close enough neighborhood, but {EM} always escapes such regions exponentially. Such one-cluster regions can be arbitrarily worse than the global minimum. Experiments show that this contrast does not only hold for mixtures of two Bernoullis, but also BMMs with any number of components. 

The rest of the paper is structured as follows. In Section \ref{background}, we give necessary background and notations. Following that, the main contributions are summarized in Section \ref{contribution}. In Section \ref{analysis}, we analyze the stability of {EM} and {GD} around one-cluster regions, for GMMs and BMMs with two components. The properties of such one-cluster regions are shown in Section \ref{struc_res}. We provide supporting experiments in Section \ref{exp} and conclude in Section \ref{sec:conclusion}.

\section{BACKGROUND AND NOTATIONS}\label{background}

A mixture model of $m$ components has the distribution:
\be\label{mixture_model}
p(\x) =  \sum_{c=1}^m \th_c f({\x| \boldsymbol \mmu_c}).
\en
with the \textit{mixing vector} ${\bm \pi}:=(\pi_1, \dots, \pi_m)$ on the $m-1$ dimensional probability simplex \cite{conop}: $\Delta_m :=\{\tth : \tth \succeq 0,\,{\bf 1}^T \tth= 1\}$.
Here $f(\cdot|\mmu_c)$ is the conditional distribution given cluster $c$, and parameters $\mmu_c$. The sample space (i.e., space of $\textbf{x}$) is $D$-dimensional.

We study the population likelihood, where there are infinitely many samples and the sample distribution is the true distribution. This assumption is common (e.g., in \cite{balakrishnan2017statistical}), and it allows us to separate estimation error from optimization error. Since we consider the problem of minimizing the negative log likelihood, the loss function is:
\be\label{loss}
\ell = -\E [\log p(\x)],
\en
where the expectation is over $p^*(\x)$, the true distribution. We assume by default this notation of expectation in this section. The loss function above is also known as cross entropy loss. 

The true distribution $p^*(\x)$ is assumed to have the same form as the model distribution: $p^*(\x) =  \sum_{c=1}^m \th^*_c f({\x| \boldsymbol \mmu^*_c})$,
where $\pi_c^* \in (0, 1)$. 
Notice that the population assumption is not necessary in this section. We could replace $p^*(\x)$ with a finite sample distribution, with i.i.d.~samples drawn from $p^*(\x)$.

\paragraph{Gaussian mixtures} For Gaussian mixtures, we consider the conditional distributions:
\be\label{GMM}
f({\x| \mmu_c})= {\mathcal N}({\x| \boldsymbol \mu_c}, I),
\en
where $\x\in \R^D$, $\mmu_c\in \R^D$ and the covariance matrix of each cluster ${\bm \Sigma}_c = I$.

\paragraph{Bernoulli mixtures} For Bernoulli mixtures, we have $\x \in \{0, 1\}^D$, $\mmu_c \in [0, 1]^D$ and:
\be\label{BMM}
f({\x| \mmu_c})= \prod_{i=1}^D B(x_{i}|\mu_{c, i}) = B({\x| \boldsymbol \mu_c}),
\en
where the Bernoulli distribution is denoted as $ B(x|\mu) = \mu^x (1-\mu)^{1-x}$.  Slightly abusing the notation, we also use $B({\x| \boldsymbol \mu_c})$ to represent the joint distribution of $x_1, \dots, x_D$ as a product of the marginal distributions. The expectation of the conditional distribution $B(\x|\mmu_c)$ is:
\be\label{conditional}
\E_{\x\sim B(\x|\mmu_c)}[\x] = \boldsymbol \mu_c,
\en
so $\boldsymbol \mu_c$ can be interpreted as the mean of cluster $c$. The covariance matrix of cluster $c$ is ${\bm \Sigma}_c={\rm diag}(\mmu_{c}*({\bf 1}-\mmu_{c}))$, where we use $*$ to denote element-wise multiplication. Also, we always assume $\mmu_1^* \in (0, 1)^D$, $\mmu_2^* \in (0, 1)^D$ and thus $\xm :=\E[\x] \in (0, 1)^D$.

\subsection{{EM} algorithm}
We briefly review the {EM}  algorithm \cite{bishop} for mixture models. Define the responsibility function\footnote{This is slightly different from the usual definition by a factor of $\pi_c$ (see, e.g., \cite{bishop}). } $\gamma_{c}(\x)$:
\be\label{response}
 \gamma_c(\x) := \frac{f(\x|\mmu_c)}{p(\x)}.
\en
 The {EM}  update map $M$ is:
\be\label{update}
M(\mmu_c) := \frac{\E[\x \gamma_c(\x)]}{\E[ \gamma_c(\x)]},\, M(\pi_c) := \pi_c \E[ \gamma_c(\x)].
\en
Sometimes, it is better to interpret \eqref{update} in a different way. Define an unnormalized distribution $\tilde{q}_c(\x) = p^*(\x)\gamma_c(\x)$. The corresponding partition function $Z_c$ and the normalized distribution $q_c(\x)$ can be written as:
\be
 Z_c = \int \tilde{q}_c(\x) d\x, \,q_c(\x) = \frac{\tilde{q}_c(\x)}{Z_c},
\en
where the integration is replaced with summation, given discrete mixture models. \eqref{update} can be rewritten as:
\be\label{update_EM}
M(\mmu_c) := \E_{\x\sim q_c(\x)}[\x],\, M(\pi_c) := \pi_c Z_c.
\en
This interpretation will be useful for our analysis near one-cluster regions in Section \ref{analysis}.

\subsection{Gradient Descent}
GD and its variants are the default algorithms in optimization of deep learning \cite{dlreview}. In mixture models, we have to deal with constrained optimization, and thus we consider projected gradient descent (PGD) \cite{pgd}. From  \eqref{mixture_model}, \eqref{loss} and \eqref{response}, the derivative over $\tth$ is:
\be
\label{GD_pi}\frac{\partial \ell}{\partial \pi_c} = -\E[\gamma_c(\x)] = -Z_c.
\en
With \eqref{GMM} and \eqref{BMM}, the derivatives over $\mmu_c$ are:
\be
\frac{\partial \ell}{\partial \mmu_c} &=&
-\pi_c \E [ \gamma_c(\x) {\bm \Sigma}_c^{-1}(\x - \mmu_c)]\\
\label{GD_mu}&=&-\pi_c {\bm \Sigma}_c^{-1} Z_c (\E_{\x\sim q_c(\x)}[\x] - \mmu_c).
\en
The equation above is valid for both GMMs and BMMs. However, for BMMs, ${\bm \Sigma}_c^{-1}$ raises numerical instability near the boundary of $[0, 1]^D$.   Therefore, we use the following equivalent formula instead in the implementation:
\be\label{grad_0_mu}
\frac{\partial \ell}{\partial \mu_{cj}} = -\pi_c \E\left[ \gamma_{c}(\x)\frac{(-1)^{1 - x_{j}}}{B(x_{j}|\mu_{cj})}\right].
\en

After a {GD} step, we have to project $(\tth, \mmu_c)$ back into the feasible space $\Delta_m \times [0, 1]^D$. For BMMs, we use Euclidean norm projection. The projection of the $\mmu_c$ part is simply:
\be\label{PGD}
P_{\mmu_c}(\mmu_c) = \min\{\max\{\mmu_c, {0}\}, {\bf 1}\},
\en
with $\min$ and $\max$ applied element-wise. Projection of $\mmu_c$ is not needed for GMMs. For the $\tth$ part, we borrow the algorithm from \cite{shalev2006efficient}, which essentially solves the following optimization problem:
\be
&&P_{\tth}(\tth) =  {\rm argmin}_{\tth'} ||\tth - \tth'||_2\\
&&\textrm{such that }\tth' \succeq 0, {\bf 1}^T \tth' = 1.
\en
 The projected gradient descent is therefore:
\be\label{gd_updates}
\tth \leftarrow P_{\tth}\left(\tth - \alpha \frac{\partial \ell}{\partial \tth}\right),\, \mmu_c \leftarrow P_{\mmu_c}\left(\mmu_c - \alpha \frac{\partial \ell}{\partial \mmu_c}\right),\tr
\en
with $\alpha$ the step size.  


\subsection{$k$-cluster region}\label{kcluster}

We define a $k$-cluster region and a $k$-cluster point as:
\begin{definition}
A $k$-cluster region is a subset of the parameter space where $||{\bm \pi}||_0 = k$, with $||\cdot||_0$ denoting the number of non-zero elements. An element in a $k$-cluster region is called a $k$-cluster point.
\end{definition} 
In this work, we focus on one-cluster regions and mixtures of two components. In Section \ref{exp}, we will show experiments on $k$-cluster regions for an arbitrary number of components. A key observation throughout our theoretical analysis and experiments, is that with random initialization, {GD} often converges to a $k$-cluster region where $k < m$, whereas {EM}  almost always escapes such $k$-cluster regions, with random initialization and even with initialization in the neighborhood of $k$-cluster regions.

Now, let us study one-cluster regions for mixture models of two components. WLOG, assume $\pi_1 = 0$. To study the stability near such regions, we consider $\pi_1 = \epsilon$, with $\epsilon$ sufficiently small. From \eqref{response}, the responsibility functions can be approximated as:
\be\label{gammas}
\gamma_1(\x) = \frac{f(\x|\mmu_1)}{f(\x|\mmu_2)},\, \gamma_2(\x) = 1.
\en
Under this approximation, $q_2(\x) = p^*(\x)$. For {EM} , $\mmu_2$ converges to $\xm$ within one step based on \eqref{update}.
For {GD}, $\mmu_2$ converges to $\xm$ at a linear rate\footnote{In the context of optimization, ``a sequence $\{x_k\}_{k=1}^{\infty}$ converges to $L$ at a linear rate" means that there exists a number $\mu\in (0, 1)$ such that $\lim_{k\to \infty}||x_{k+1} - L||/||x_k - L|| =\mu$.} based on \eqref{GD_mu}\footnote{For BMMs, assume $\mmu_2$ is not initialized on the boundary. The rate is upper bounded by the initialization as $\mmu_2$ converges to $\xm \in (0, 1)^D$.}. 

For convenience, we define $\b :=\mmu_1 - \mmu_2$ the difference between $\mmu_1$ and $\mmu_2$.

\subsection{Stable fixed point and stable fixed region}
A fixed point $p$ is stable under map $M$ if there exists a small enough feasible neighborhood $B$ of $p$ such that for any $p'\in B$, $\lim_{n\to \infty} M^n(p') = p$, where we use $M^n$ to denote a composition for $n$ times. Otherwise, the fixed point is called unstable. Similarly, a fixed region $R$ is stable under map $M$ if there exists a small enough feasible neighborhood $B$ of $R$ such that for any $p'\in B$, $\lim_{n\to \infty} M^n(p')$ exists and $\lim_{n\to \infty} M^n(p') \in R$. We will use these definitions in Section \ref{analysis}.

\section{MAIN CONTRIBUTIONS}\label{contribution}

We summarize the main contributions in this paper, starting from mixtures of two Gaussians. This is a well-known model that has been widely studied. However, we are not aware of any result regarding $k$-cluster points. As a first result in this line of research, we show that {EM} is better than {GD} in escaping one-cluster regions:
\begin{result}
Consider a mixture of two Gaussians with $\pi_1^*\in (0, 1)$, unit covariance matrices and true distribution
\be
p^*(\x) = \pi^*_1 \N(\x|\mmu^*, I) + \pi^*_2 \N(\x|-\mmu^*, I)
\en
When $\pi_1 = \epsilon$ is initialized to be small enough, $\mmu_2 = \xm$ and $\b^T \mmu^* \neq 0$, {EM}  increases $\pi_1$ exponentially fast, while {GD} increases $\pi_1$ linearly\footnote{Here, we use the usual notion of function growth. Do not confuse it with the rate of convergence.}.
\end{result}
This result indicates that {EM}  escapes the neighborhood of one cluster-regions faster than {GD} by increasing the probability $\pi_1$ of cluster 1 at a rate that is exponentially faster than {GD}.  The escape rates of {EM}  and {GD} are formally proven in Theorems~\ref{balance_EM} and~\ref{balance_GD}.

Our second result concerns mixtures of two Bernoullis:

\begin{result}\label{cont_BMM}
For mixtures of two Bernoullis, there exist one-cluster regions that can trap {GD}, but do not trap {EM}. If any one-cluster region traps {EM}, it will also trap GD.
\end{result}

This result is formally stated in Theorems~\ref{trap_EM_trap_GD} and~\ref{unstable_proof}. It shows that {EM} is also better than {GD} in escaping one-cluster regions in BMMs. Our third result concerns the value of one-cluster regions, which is an informal summary of Theorems \ref{min_D2} and \ref{no_global}:
\begin{result}
The one-cluster regions stated in Result \ref{cont_BMM} are local minima, and they can be $O(D)$ worse than the optimal value.
\end{result}

So far, we have seen that {EM} is better than {GD} in mixtures of two components, in the sense that {EM} escapes one-cluster regions exponentially faster than {GD} in GMMs, and that {EM} escapes local minima that trap {GD} in BMMs. Empirically, we show that this is true in general. In Section \ref{exp}, we find that for BMMs with an arbitrary number of components and features, when we initialize the parameters randomly,  {EM} always converges to an $m$-cluster point, i.e., all clusters are used in the model. Comparably, {GD} converges to a $k$-cluster point with a high probability, where only some of the clusters are employed in fitting the data (i.e., $k<m$). This result, combined with our analysis for mixtures of two components, implies that {EM} can be more robust than {GD} in terms of avoiding certain types of bad local optima when learning mixture models.

\section{ANALYSIS NEAR ONE-CLUSTER REGIONS}\label{analysis}

In this section, we analyze the stability of {EM} and {GD} near one-cluster regions. Our theoretical results for EM and GD are verified empirically in Section \ref{exp}.

\subsection{Mixture of two Gaussians}\label{Sec: GMM}

We first consider mixtures of two Gaussians with identity covariance matrices, under the infinite-sample assumption. Due to translation invariance, we can choose the origin to be the midpoint of the two cluster means, such that: $\mmu_1^* = \mmu^*$ and $\mmu_2^* = -\mmu^*$. 

\subsubsection{{EM} algortihm}

As long as $\pi_1$ is sufficiently small, the following theorem shows that {EM} will increase $\pi_1$ and therefore will not converge to a one-cluster solution.  Recall from (\ref{update_EM}) that $\pi_1$ is multiplied by $Z_1$ at every step of {EM} and therefore we show that $Z_1>1$ almost everywhere, ensuring that $\pi_1$ will increase.
Since $\mmu_2$ converges to $\xm$ within one step, we only consider those points where $\mmu_2 = \xm$.

We can show that under mild assumptions,  {EM} escapes one-cluster regions exponentially fast:
\begin{theorem}\label{balance_EM}
Consider a mixture of two Gaussians with $\pi_1^*\in (0, 1)$, unit covariance matrices and true distribution
\be
p^*(\x) = \pi^*_1 \N(\x|\mmu^*, I) + \pi^*_2 \N(\x|-\mmu^*, I)
\en
When $\pi_1 = \epsilon$ is initialized to be small enough, $\mmu_2 = \xm$ and $\b^T \mmu^* \neq 0$, {EM} increases $\pi_1$ exponentially fast.
\end{theorem}
\begin{proof}
It suffices to show that $Z_1 > 1$ and $Z_1$ grows if initially $\b^T \mmu^* \neq 0$. To calculate $Z_1$, we first compute $\tilde{q}_1(\x) = p^*(\x)\gamma_1(\x)$:
\be\label{tilde_q_2}
\tilde{q}_1(\x) &=& \pi_1^* e^{\b^T(\mmu^* - \mmu_2)}\N(\x| \mmu^* +\b, I) \tr
&+& \pi_2^* e^{-\b^T(\mmu^* + \mmu_2)} \N(\x| -\mmu^* + \b, I).\tr
\en
This equation shows that $\tilde{q}_1(\x)$ corresponds to an unnormalized mixture of Gaussians with their means shifted by $b$ and their mixing coefficients rescaled in comparison to $p^*(\x)$.
If $\b = {\bf 0}$, then $\tilde{q}_1(\x) = p^*(\x)$. So, $\b$ describes how different $\tilde{q}_1(\x)$ deviates from $p^*(\x)$. The partition function can be computed by integrating out $\tilde{q}_1(\x)$:
\be
Z_1 =  \pi_1^* e^{\b^T(\mmu^* - \mmu_2)} + \pi_2^* e^{-\b^T(\mmu^* + \mmu_2)}.
\en
In fact, $\mmu_2 = \overline{\x} = (\pi_1^* - \pi_2^*)\mmu^*$ which can be derived from $p^*(\x)$. So, $Z_1$ becomes:
\be\label{Z_q}
Z_1 =  \pi_1^* e^{2\pi_2^* \b^T \mmu^*} + \pi_2^* e^{-2\pi_1^* \b^T \mmu^*}.
\en
Using the fact $e^x \geq 1 + x$ and that equality holds iff $x = 0$, we can show that when $\b^T \mmu^* \neq 0$, $Z_1 > 1$. 

Now, let us show that $Z_1$ increases. It suffices to prove that $|\b^T \mmu^*|$ increases. From \eqref{tilde_q_2} and \eqref{update_EM}, we have the update equation for $\mmu_1$: $\mmu_1 \leftarrow (\pi'_1 - \pi'_2)\mmu^* + \b$,
 with
\be\label{pi1prime}
\pi'_1 = \pi_1^* Z_1^{-1}  e^{2\pi_2^* \b^T \mmu^*}, \pi'_2 = \pi_2^* Z_1^{-1}  e^{-2\pi_1^* \b^T \mmu^*},
\en
If $\b^T \mmu^* > 0$, then $\pi'_1 > \pi_1^*$ and $\pi'_2 < \pi_2^*$. So, $\mmu_1^{(t+1)} = \mmu_1^{(t)} + \delta \mmu^*$ with $\delta = \pi'_1-\pi_1^* + \pi_2^* - \pi'_2> 0$, and
\be\label{bt+1}
(\b^{(t+1)})^T \mmu^* = (\b^{(t)})^T \mmu^* + \delta ||\mmu^*||_2^2 > (\b^{(t)})^T \mmu^*,
\en
where we use the superscript $(t+1)$ to denote the updated values and $(t)$ to denote the old values. 

Similarly, we can prove that $\b^T \mmu^* $ will decrease if $\b^T \mmu^* < 0$. Hence, from \eqref{Z_q}, $Z_1$ increases under {EM} if initially $\b^T \mmu^* \neq 0$.
\end{proof}
This theorem above can be extended to any mixture of two Gaussians with known fixed covariance matrices $\Sigma$ for both clusters, using the transformation $\mmu_i\to \Sigma^{-1/2} \mmu_i$ and $\mmu^* \to \Sigma^{-1/2}\mmu^*$. We will prove it formally in Appendix \ref{2-gmm-gen}.

One may wonder what happens if originally $\b^T \mmu^* = 0$. If we choose $\mmu_1$ randomly, $\b^T \mmu^* = 0$ happens with probability zero since the corresponding Lebesgue measure is zero. Moreover, in numerical calculation such points are extremely unstable and thus unlikely. A similar condition appears in balanced mixtures of Gaussians as well, e.g., Theorem 1 in \cite{xu2016global}.

\paragraph{Rotation} The proof of Theorem \ref{balance_EM} also shows an interesting phenomenon: $\mmu_1$ rotates towards $\mmu^*$ or $-\mmu^*$ (see Fig.~\ref{fig:rotation}). 
In fact, one can show that if $(\b^{(t)})^T \mmu^* > 0$, then:
\be
\frac{\langle \mmu_1^{(t+1)} , \mmu^* \rangle}{||\mmu_1^{(t+1)}|| \cdot ||\mmu^*|| }
\geq \frac{\langle \mmu_1^{(t)}, \mmu^* \rangle }{||\mmu_1^{t}|| \cdot ||\mmu^*||},
\en
where the equality holds iff $\mmu_1$ is a multiple of $\mmu^*$. A similar result can be shown for $(\b^{(t)})^T \mmu^* < 0$.

\begin{figure}
    \centering
    \includegraphics[width=6cm]{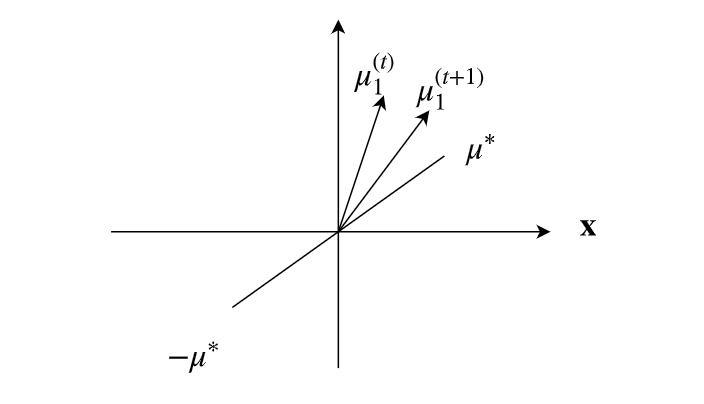}
    \caption{Illustration of an {EM} update as a rotation of $\mmu_1^{(t)}$ to $\mmu_1^{(t+1)}$ towards $\mmu^*$.}
    \label{fig:rotation}
\end{figure}

\paragraph{Effect of Separation} Another interesting observation is that one can understand how the degree of separation affects the escape rate. If $||\mmu^*||$ is small, from \eqref{Z_q}, $Z_1$ will be close to $1$. Also, from \eqref{pi1prime} and \eqref{bt+1}, $\b^T \mmu^*$ will change more slowly, adding to the effect of slowing down the escape rate.

\subsubsection{GD algorithm}\label{GD analysis}

Now, let us analyze the behavior of {GD} near one-cluster regions. After one {GD} iteration, the mixing coefficients become $ (\pi_1+\alpha Z_1, \pi_2+\alpha)$
before projection (where $\pi_2+\alpha Z_2 = \pi_2+\alpha$ since $Z_2=1$). Assume a small step size $\alpha$ such that $\pi_2 + \alpha > \pi_1 + \alpha Z_1$. After projection based on \eqref{gd_updates}, the updated mixing coefficients are:
\be\label{PGD_mix}
(\pi_1 + ({\alpha}/{2})(Z_1 - 1), \pi_2 - ({\alpha}/{2})(Z_1 - 1)).
\en
Combining with \eqref{update_EM}, whenever $Z_1 > 1$, both {EM} and {GD} increase $\pi_1$. In this sense, {EM} and {GD} achieve an agreement. However, in {GD} $\mmu_1$ changes little since the change is proportional to $\pi_1$. This argument is true for any mixture models with two components.

For GMMs, the update of $\mmu_2$ is:
\be\label{GMM_mu2}
\mmu_2 \leftarrow \mmu_2 + \alpha (\xm - \mmu_2).
\en
Hence, $\mmu_2^{(t)}$ is on the line segment between $\xm$ and $\mmu_2^{(0)}$, and $\mmu_2$ converges to $\xm$ at a linear rate $1-\alpha$. 

If $\mmu_2 \neq \xm$, it is possible to have $Z_1 < 1$. For example, take $D = 1$, $\mu_2 = 3$, $\mu_1 = 5$, $\pi_1^* = \pi_2^* = 1/2$ and $\mu^* = 2$, we have $Z_1 = (1/2)(e^{-10}+e^{-2}) < 1$. In such cases, $\pi_1$ could either increase or decrease.

In the worst case, $\pi_1$ stays small until $\mmu_2$ converges to $\xm$. Then, from \eqref{Z_q}, $Z_1 > 1$ given $\b^T \mmu^* \neq 0$. Assuming $\mmu_1$ changes little in this process, we can conclude that with probability one {GD} escapes one-cluster regions. However, the growth of $\pi_1$ is linear compared to {EM}. Therefore, we have the following theorem:
\begin{theorem}\label{balance_GD}
For mixtures of two Gaussians and arbitrary initialization $(\mmu_1, \mmu_2)$, there exists a $\pi_1$ small enough such that when the mixture is initialized at $(\pi_1, \mmu_1, \mmu_2)$, {GD} increases $\pi_1$ as a linear function of the number of time steps.
\end{theorem}
\begin{proof}
For $\pi_1$ small enough, $\mmu_1$ stays in a small neighborhood, and $\mmu_2$ converges to $\xm$ according to \eqref{GMM_mu2}. Around $\mmu_2 = \xm$, $Z_1$ stays close to a constant and $Z_1 \geq 1$. The linear growth of $\pi_1$ follows from \eqref{PGD_mix}. 
\end{proof}

Notice that \eqref{PGD_mix} is true for both GMMs and BMMs. The only difference is the computation of $Z_1$. From this equation, we know that if any one-cluster region traps {EM}, then $Z_1 < 1$ and $\mmu_2 = \xm$. On this configuration, for the {GD} algorithm, we know that $\pi_1$ decreases if it is not zero, and $\mmu_1, \mmu_2$ stays the same. Hence, this region traps {GD} as well. So, the following theorem holds:

\begin{theorem}\label{trap_EM_trap_GD}
For two-component mixture models, if any one-cluster region is stable for {EM}, it is stable for GD.
\end{theorem}
\begin{proof}
Near any one-cluster region that is stable for {EM}, $\mmu_2 = \xm$ and from \eqref{update_EM}, $Z_1 < 1$. Also, $\mmu_1$ changes very little since the update is proportional to $\pi_1$. So, $Z_1$ is a constant and $\pi_1$ decreases linearly according to \eqref{PGD_mix}. The parameters stop changing at $\pi_1 = 0$.
\end{proof}


\subsection{Mixture of two Bernoullis}\label{BMM2}

Now, let us see if Bernoulli mixtures have similar results. With the {EM} algorithm, $\mmu_2 = \overline{\x}$ can be easily obtained. From the definition of $\tilde{q}_c(\x)$, we have:
\be\label{qtildeBMM}
\tilde{q}_1(\x) &=& p^*(\x)\frac{B(\x|\mmu_1)}{B(\x|\xm)} \tr
& =& \pi_1^* \prod_{i=1}^D \left( \mu_{1i}^* \frac{\mu_{1i}}{\overline{x}_i} \right)^{x_i} \left( (1-\mu_{1i}^*) \frac{1-\mu_{1i}}{1-\overline{x}_i} \right)^{1-x_i}\tr
&+&\pi_2^* \prod_{i=1}^D \left( \mu_{2i}^* \frac{\mu_{1i}}{\overline{x}_i} \right)^{x_i} \left( (1-\mu_{2i}^*) \frac{1-\mu_{1i}}{1-\overline{x}_i} \right)^{1-x_i}.\tr
\en
Consider the unormalized binary distribution $B(x|\mu'_1, \mu'_2) = (\mu'_1)^x (\mu'_2)^{1-x}$ with $\mu'_1 + \mu'_2 \neq 1$.  Then the normalization factor (partition function) is $\mu'_1 + \mu'_2$. With this fact, one can derive that
\be
Z_1 &=& \pi_1^* \prod_{i=1}^D \left( \mu_{1i}^* \frac{\mu_{1i}}{\overline{x}_i}  + (1-\mu_{1i}^*) \frac{1-\mu_{1i}}{1-\overline{x}_i} \right)\tr
&+& \pi_2^* \prod_{i=1}^D \left( \mu_{2i}^* \frac{\mu_{1i}}{\overline{x}_i}  + (1-\mu_{2i}^*) \frac{1-\mu_{1i}}{1-\overline{x}_i} \right)\tr
\en
From the intuition of Gaussian mixtures, we similarly define $2{\bf \mmu}^* = \mmu_1^* - \mmu_2^*$ as the separation between the two true cluster means. As usual, ${\bf b}$ describes the separation between $\mmu_1$ and $\mmu_2$. With the notations above, we obtain:
\be
\label{part_func}
Z_1 = \pi_1^* \prod_{i=1}^D \left( 1 + \pi_2^* \lambda_i \right) +  \pi_2^* \prod_{i=1}^D \left( 1 - \pi_1^* \lambda_i \right),
\en
with
\be\label{lambda}
\lambda_i = 2 S_i^{-1}{\mu^*_i}b_i\mbox{ and }S_i ={\rm var}[x_i] = \overline{x}_i (1 - \overline{x}_i).
\en
Approximating $1+x\sim e^x$ and taking the covariance matrix to be the identity, we retrieve GMM (see \eqref{Z_q}). 

The update of ${\bm \l}$ can be computed from \eqref{update_EM}:
\be\label{update_EM_gen}
M({\bm \lambda})_i = \l_i + (2S_i^{-1}\mu_i^*)^2 \pi_1^* \pi_2^* \frac{\Lambda_i}{Z_1}(B_{1i}-B_{2i}),
\en
where we denote $\Lambda_i = \mu_{1i}(1-\mu_{1i})$ and 
\be
\label{B1iB2i}B_{1i} :=\prod_{j\neq i}^D (1+\pi_2^* \lambda_j), B_{2i} :=\prod_{j\neq i}^D (1-\pi_1^* \lambda_j).
\en
The derivation of \eqref{update_EM_gen} is presented in Appendix \ref{315}.
 Notice that each $\lambda_i$ is an affine transformation of $\mu_{1i}$, as defined by \eqref{lambda}. Hence, the domain of ${\bm \l}=(\l_1, \dots, \l_D)$ is a Cartesian product of closed intervals. We assume by default that ${\bm \l}$ is in the domain. Studying the update of ${\bm \lambda}$ instead of $\mmu_1$ is more convenient, as we will see in the following subsections that ${\bm \l}$ plays an important role in expressing the convergence results. 
 
 Also, we denote the covariance between features $i$ and $j$:
\be
\sigma_{ij} = \E[x_i x_j] - \E[x_i]\E[x_j] = 4\pi_1^* \pi_2^* \mu_i^* \mu_j^*.
\en
We consider the case when there are no independent pairs, i.e., $\sigma_{ij}\neq 0$ for any $i\neq j$. 

\subsubsection{Attractive one-cluster regions for GD}

 When $Z_1 < 1$ and $\mmu_2 = \xm$, {GD} will get stuck in one-cluster regions, as shown in the proof of Theorem \ref{trap_EM_trap_GD}. If $\mmu_2 \neq \xm$, $\mmu_2$ converges to $\xm$ at a linear rate:
 \be
 \mmu_2 \leftarrow \mmu_2 + \alpha {\bm \Sigma}_2^{-1}(\xm - \mmu_2),
 \en
 as can be seen from \eqref{GD_mu}. Hence, we assume that $\mmu_2$ has already converged to $\xm$.  With \eqref{part_func}, we can define attractive one-cluster regions for GD:
 \begin{proposition}[\textbf{attractive one-cluster regions for GD}]\label{trap_GD}
 Denote $T := \{{\bm \l}| Z_1({\bm \l}) < 1\}$.
 The one-cluster regions $(\pi_1, \mmu_1, \mmu_2)$ with ${\bm \l}(\mmu_1)\in T$, $\mmu_2 = \xm$ and $\pi_1 = 0$ are attractive for {GD}, given small enough step size. Here, we denote ${\bm \l}(\mmu_1)$ as an affine function of $\mmu_1$, in the form of \eqref{lambda}.
 \end{proposition}
 
 For example, when $D = 2$, from \eqref{part_func}, we have $Z_1 = 1 + \pi_1^* \pi_2^* \lambda_1 \lambda_2$. If $\lambda_1 \lambda_2 < 0$, then $Z_1 < 1$, and {GD} will get stuck. If $\lambda_1 \lambda_2 > 0$, $Z_1 > 1$, and then {GD} will escape one-cluster regions. In the latter case, ${\bm \l}\in \R_{++}^2$ or  ${\bm \l}\in -\R_{++}^2$.

\subsubsection{Positive regions}\label{pos_reg}

Now, let us study the behavior of {EM}. We will show that {EM} escapes one-cluster regions when ${\bm \l} \in \R_{++}^2 \cup (-\R_{++}^2)$, for the two-feature case. More generally with $D$ features, we should consider ${\bm \l} \in \R_{++}^D \cup (-\R_{++}^D)$.

In the field of optimization, both $\R^D_{++}$ and $-\R^D_{++}$ are known as proper cones and a relevant order can be defined. They are critical in our analysis, so, we define these two cones as positive regions:
\begin{definition}[\textbf{positive regions}]
A positive region of ${\bm \l}$ is defined to be one of $\{\R^D_{++}, -\R^D_{++}\}$. We denote it as
$P := \R^D_{++} \cup (-\R^D_{++})$.
\end{definition}

The positive regions are interesting because $\R^D_{++}$ ensures that each element $\mmu_1 - \mmu_2$ has the same sign as $\mmu_1^* - \mmu_2^*$ and $-\R^D_{++}$ ensures that each $\mmu_1 - \mmu_2$ has the opposite sign of $\mmu_1^* - \mmu_2^*$.

\begin{theorem}\label{pos_Z}
For any number of features, $Z_1({\bm \l}) > 1$ in the positive regions. 
\end{theorem}
\begin{proof}
From \eqref{part_func},
\be
\nabla Z_1({\bm \l}) = \pi_1^* \pi_2^* \left( \prod_{i=1}^D \left( 1 + \pi_2^* \lambda_i \right) - \prod_{i=1}^D \left( 1 - \pi_1^* \lambda_i \right)\right).\nonumber
\en
So, $\nabla Z_1({\bm \l})\succ 0$ if ${\bm \l} \succ 0$. Denote $g(t) = Z_1(t{\bm \l})$, then $Z_1({\bm \l}) - 1 = g(1) - g(0) =\int_0^1 {\bm \l}^T \nabla Z_1(t{\bm \l}) dt > 0$, given ${\bm \l} \succ 0$. we can similarly prove $Z_1({\bm \l}) > 1$ given ${\bm \l} \prec 0$. 
\end{proof}

We also notice that positive regions are stable. Define $M({\bm \l})$ to be the {EM} update of ${\bm \l}$ based on the {EM} update of $\mmu_1$ and \eqref{lambda}.  For any number of features, the following lemma holds:

\begin{lemma}[\textbf{stability of positive regions}]\label{stable_pos}
For any number of features, given $\pi_1 = \epsilon$ small, the two positive regions $\R_{++}^D$ and $-\R_{++}^D$ are stable for {EM}, i.e., $M({\bm \lambda}) \succ {\bm \lambda}$ for all ${\bm \lambda} \succ 0$ and $M({\bm \lambda}) \prec {\bm \lambda}$ for all ${\bm \lambda} \prec 0$.
\end{lemma}
\begin{proof}
If ${\bm \l} \succ 0$, then $B_{1i} > B_{2i}$ for any $i$ (based on \eqref{update_EM_gen} and \eqref{B1iB2i}), and thus $M({\bm \l}) \succ {\bm \l}$.  A similar argument demonstrates that $M({\bm \l}) \prec {\bm \l}$ when ${\bm \l} \prec 0$.
\end{proof}

From the lemma above, we conclude that in the positive regions, $|\b^T \mmu^*|$ will increase, similar to  Theorem \ref{balance_EM}. Lemma \ref{stable_pos} and Theorem \ref{pos_Z} lead to the following:

\begin{corollary}\label{pos_cor}
For any number of features, if ${\bm \l}$ is initialized in a positive region, then {EM} will escape one-cluster regions exponentially fast.  \end{corollary} 
\begin{proof}
WLOG, we assume ${\bm \l}\succ 0$. From Theorem \ref{pos_Z}, at positive regions, $Z_1 > 1$. Define $g(t) = Z_1({\bm \l} + t (M({\bm \l}) - {\bm \l}))$. $Z_1$ increases after an {EM} update because:
\be
&&Z_1( M({\bm \l}) ) - Z({\bm \l}) = g(1) - g(0) \tr &=& \int_0^1 (M({\bm \l}) - {\bm \l})^T \nabla Z_1({\bm \l} + t (M({\bm \l}) - {\bm \l})) dt > 0,\nonumber
\en
where we use $M({\bm \lambda}) \succ {\bm \l}$ from Lemma \ref{stable_pos} and $\nabla Z_1({\bm \l})\succ 0$ if ${\bm \l} \succ 0$ from the proof of Theorem \ref{pos_Z}. From \eqref{update_EM}, $\pi_1$ increases exponentially. 
\end{proof}

With {GD}, we have similar results as Corollary \ref{pos_cor}, but {GD} escapes one-cluster regions much more slowly, which can be derived in a similar way as in Section \ref{Sec: GMM}.

 What if ${\bm \l}$ is not initialized in positive regions? The following theorem tells us that if $D=2$, no matter where ${\bm \l}$ is initialized, {EM} will almost always escape the one-cluster regions.

\begin{theorem}\label{EM_GD_22}
For $m = D=2$, given $\sigma_{12} \neq 0$ and $\xm \in (0, 1)^D$, with the {EM} algorithm, $\pi_1 = \epsilon$, $\mmu_2 = \xm$ and uniform random initialization for $\mmu_1$, ${\bm \l}$ will converge to the positive regions at a linear rate with probability $1$.  Therefore, {EM} will almost surely escape one-cluster regions. 
\end{theorem}

\begin{proof}
Here, we give a proof sketch, and the more detailed proof of Theorem \ref{EM_GD_22} can be found in Appendix \ref{sec: A.2}.
First, 
in the worst case, $||{\bm \l}||$ shrinks to a neighborhood of the origin. Then, near the origin, {EM} rotates ${\bm \l}$ towards positive regions. 
\end{proof}
For general $m$ and $D$, we conjecture that {EM} almost always escapes $k$-cluster regions where $k < m$, as described in Appendix \ref{conj}.

\subsubsection{EM as an ascent method}
So far, we have studied positive regions $P$. We showed some nice properties of $P$ and proved that ${\bm \l}$ converges to $P$ almost everywhere when $D = 2$. In this section, we show that {EM} can be treated as an ascent method for $Z_1$:

\begin{theorem}\label{EM_AS_AS}
For ${\bm \lambda}$ in the feasible region, we have $\nabla Z_1({\bm \lambda})^T (M({\bm \lambda}) - {\bm \lambda}) \geq 0$, with equality holds iff ${\bm \lambda} = 0$. 
\end{theorem}
\begin{proof}
From \eqref{part_func}, $\nabla_i Z_1 ({\bm \l}) = \pi_1^* \pi_2^* (B_{1i} - B_{2i})$.
So, from \eqref{update_EM_gen}, for any $i$, $\nabla_i Z_1 ({\bm \l})(M({\bm \l})_i - \l_i)$ is nonnegative. Equality holds iff for all $i\in [D]$, $B_{1i} = B_{2i}$. It follows that for any $i$, $1 +\pi_1^* \lambda_i = 1 - \pi_2^* \lambda_i$. Hence, equality holds iff ${\bm \lambda} = 0$.
\end{proof}

With this theorem, we can take a small step in the {EM} update direction: ${\bm \l} + \alpha (M({\bm \l}) -{\bm \l})$, with $\alpha$ small. In this way, we can always increase $Z_1$ until $Z_1 > 1$ and {EM} escapes the one-cluster regions. 

\subsection{EM vs GD}
Now, we are ready to prove one of the main results in our paper: for mixtures of two Bernoullis, there exist one-cluster regions that can trap {GD} but not {EM}. We first prove a lemma similar to Lemma \ref{stable_pos}. 
\begin{lemma}\label{boundary_lemma32}
For all ${\bm \l} \neq 0$ and ${\bm \l} \succeq 0$, $M({\bm \lambda})\succ 0$ and $M({\bm \l})\succeq {\bm \l}$. Similarly, for all ${\bm \l} \neq 0$ and ${\bm \l} \preceq 0$, $M({\bm \lambda})\prec 0$ and $M({\bm \l})\preceq {\bm \l}$.
\end{lemma}
\begin{proof}
WLOG, we prove the lemma for ${\bm \l} \succeq 0$. Notice that for any ${\bm \l} \succeq 0$, $M({\bm \l})\succeq {\bm \l}$ can be read directly from \eqref{update_EM_gen}. For $\l_i = 0$, $B_{1i} > B_{2i}$ since ${\bm \l} \neq 0$. Hence, $M({\bm \l})\succ 0$.
\end{proof}
This lemma leads to a main result in our work: for mixtures of two Bernoullis, there exist one cluster regions with nonzero measure that trap {GD}, but not {EM}.
\begin{theorem}\label{unstable_proof}
For mixtures of two Bernoullis, given $\pi_1^* \in (0, 1)$, $\xm \in (0, 1)^D$ and $||\mmu^* ||_0 = D$, there exist one-cluster regions $B$ that trap {GD}, but {EM} escapes such one-cluster regions exponentially fast.
\end{theorem}

\begin{proof}
Denote ${\bf e}_i$ as the $i^{\rm th}$ unit vector in the standard basis. For any $i\in [D]$, at ${\bm \l} = {\l}_i {\bf e}_i$ and $\l_i > 0$, from Lemma \ref{boundary_lemma32} and Theorem \ref{pos_Z}, we have $M({\bm \l}) \succ 0$ and $Z_1(M({\bm \l})) > 1$. Notice also that $Z_1({\bm \l}) = 1$ and $\nabla_j Z_1({\bm \l}) > 0$ for all $j \neq i$. Therefore, moving in the opposite direction of the gradient, it is possible to find a neighborhood $B$ of ${\bm \l}$, such that for all ${\bm \l}'\in B$, $Z_1({\bm \l}') < 1$ and $Z_1(M({\bm \l}')) > 1$. This region traps {GD} due to Proposition \ref{trap_GD}, and {EM} escapes such one-cluster regions exponentially fast due to Corollary \ref{pos_cor}.
\end{proof}

\section{ONE-CLUSTER LOCAL MINIMA }\label{struc_res}

It is possible to show that some one-cluster regions are local minima, as shown in the following theorem:

\begin{theorem}\label{min_D2}
The attractive one-cluster regions for {GD} defined in Proposition \ref{trap_GD} are local minima of the cross entropy loss $\ell = -\E [\log \left( \pi_1 B(\x|\mmu_1) + \pi_2 B(\x|\mmu_2)\right)]$. 
\end{theorem}

\begin{proof}
Impose $\pi_2 = 1 - \pi_1$ and treat $\ell$ as a function of $(\pi_1, \mmu_1, \mmu_2)$. Consider any small perturbation $(\delta\pi_1, \delta \mmu_1, \delta \mmu_2)$. If $\delta\pi_1 = 0$, then the loss will not decrease as $\mmu_2 = \xm$ is a local minimum of $-\E [\log  B(\x|\mmu_2)]$. If $\delta\pi_1 \neq 0$, the change of $\ell$ is determined by the first order. Only ${\partial \ell}/{\partial \pi_1}$ is nonzero, so the change is dominated by the first order:
\be
\frac{\partial \ell}{\partial \pi_1} \delta \pi_1 = (1 - Z_1|_{\mmu_1 = \mmu_1 + \delta \mmu_1})\delta \pi_1 > 0,
\en
when $(\delta\pi_1, \delta \mmu_1, \delta \mmu_2)$ is small enough. Therefore, we conclude that $(\pi_1, \mmu_1, \mmu_2)$ is a local minimum.
\end{proof}

We call such minima one-cluster local minima. The following theorem shows that they are not global minima.

\begin{theorem}\label{no_global}
Assume $\mmu^* \neq 0$ and $\pi_1^* \in (0, 1)$. For mixtures of two Bernoullis, one-cluster local minima cannot be global. The gap between the one-cluster local minima and the global minimum could be as large as ${\Theta}(D)$.
\end{theorem}

\begin{proof}
The global minimum is obtained when $p(\x) = p^*(\x)$. Denote the optimal value as $\ell^*$. We have:
\be
\ell^* -\ell_1 = -{\rm KL}(p^*(\x)|\prod_{i=1}^D p^*(x_i)) \leq 0,
\en
with ${\rm KL}(p||q)$ the Kullback-Leibler divergence. Equality holds iff $p^*(\x) = \prod_{i=1}^D p^*(x_i)$, which means that the features are mutually independent. A direct consequence is that the features are pairwise independent. Nevertheless, we assumed at the end of Section~\ref{BMM2} that $\sigma_{ij} \neq 0$ for any pair $(i, j)$. This is a contradiction.

Moreover, the difference could be very large. For example, taking $\pi_1^* = 1/2$, $\mmu_1^* = {\bf 1}$ and $\mmu_2^* = 0$, the difference is $\ell^* -\ell_1 = -(D-1)\log 2$.
\end{proof}

\section{EXPERIMENTS}\label{exp}


We analyzed the differences in behavior between {EM} and {GD} in Section \ref{analysis} for mixtures of two components. For Gaussian mixtures with two components, we have done some experiments for the $D=2$ case with $10,000$ samples to approximate the population case, which can be shown in Figure \ref{fig:gmm}. From this figure, we see that {EM} increases $\pi_1$ exponentially and {GD} increases $\pi_1$ linearly, when we initialize near a one-cluster region. 

\begin{figure}
    \centering
    \includegraphics[width=6cm]{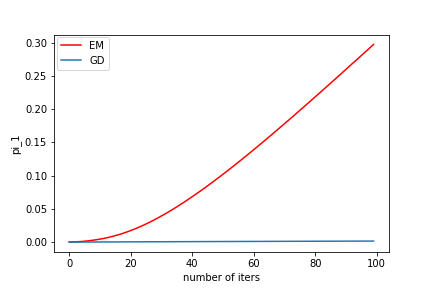}
    \caption{Initialized from $\pi_1 = 10^{-4}$, EM escapes the $1$-cluster point exponentially while GD escapes linearly with step size $0.01$, achieving $0.0015$ in $100$ steps.}
    \label{fig:gmm}
\end{figure}

For Bernoulli mixtures, we aim to understand, at least empirically, how often {EM} and {GD} converge to a $k$-cluster point with $k < m$ with an arbitrary number of components. Since $k$-cluster points only utilize some of the components, this typically implies that the fixed points are bad. We will show that {EM} never converges to such bad critical points, whereas {GD} does.


We experiment with the infinite sample case for mixtures of Bernoullis of $m$ components with $D$ binary variables, where both $m$ and $D$ take values in $\{2,3,4,5,6\}$ for a total of 25 combinations. For each combination $(m,D)\in\{2,3,4,5,6\}^2$, We run (i) {EM} until convergence with at most $20000$ iterations; (ii) {GD} until convergence with a step size of $0.02$ and at most $10000$ iterations, for $60$ times each from random initialization. We find experimentally that {EM} always converges to an $m$-cluster point that is reasonably good compared to the optimal value (greater than $0.999$ in terms of the likelihood ratio). On the other hand, {GD} is much more likely to converge to a $k$-cluster point with $k<m$, especially as $m$ or $D$ increase. Table \ref{table1} summarizes our empirical findings for GD.

\begin{table}[ht]
\centering
\begin{tabular}{c|ccccccc}
\hline
& $D=2$ & $D=3$ & $D=4$ & $D=5$ & $D=6$ \\
\hline
$m=2$& 21.7 & 13.3 & 10.0 & 21.7 & 18.3 \\
$m=3$& 46.7 & 15.0 & 56.7 & 26.7 & 26.7 \\
$m=4$& 18.3 & 26.7 & 30.0 & 36.7 & 48.3 \\
$m=5$& 48.3 & 56.7 & 50.0 & 58.3 & 63.3 \\
$m=6$& 35.0 & 46.7 & 68.3 & 58.3 & 53.3 \\
\hline
\end{tabular}
\caption{Average fraction of times (as a percentage) when {GD} converges to a $k$-cluster point with $k<m$ from random initialization. Infinite sample case. }
\label{table1}
\end{table}

We also experiment with datasets of $2000$ samples generated from mixtures of Bernoullis of $m$ components with $D$ binary variables. We choose the same hyperparameters as in the previous experiment for {EM} and {GD}.  As before, {EM} always converges to $m$-cluster points. The results for {GD} are summarized in Table \ref{table2}. From this table, we see the similarity between infinite sample and finite sample cases. Hence, it is possible to extend our analysis for the population likelihood to a likelihood with a finite number of samples.

\begin{table}[ht]
\centering
\begin{tabular}{c|ccccccc}
\hline
& $D=2$ & $D=3$ & $D=4$ & $D=5$ & $D=6$\\
\hline
$m=2$& 10.0 &  5.0 & 0.0 & 13.3 & 10.0  \\
$m=3$& 10.0 & 13.3 & 21.6 & 36.7 & 40.0 \\
$m=4$& 33.3 & 35.0 & 10.0 & 36.7 & 56.7 \\
$m=5$& 33.3 & 46.7 & 25.0 & 40.0 & 70.0 \\
$m=6$& 43.3 & 50.0 & 58.3 & 60.0 & 80.0 \\
\hline
\end{tabular}
\caption{Average fraction of times (as a percentage) when {GD} converges to a $k$-cluster point with $k<m$ from random initialization. Finite sample case.}
\label{table2}
\end{table}


A remaining question is about the quality of the $k$-cluster points that {GD} converges to. Denote the loss of the $k$-cluster point that {GD} converges to as $\ell_{\rm GD}$, and the optimal loss as $\ell^*$. In Table \ref{table3}, we compute $\exp(\ell^* - \ell_{\rm GD})$ as the ratio of the likelihood of the $k$-cluster points that {GD} converges to, to the optimal likelihood. We randomly choose 10 true distributions, compute the ratio and take the average, as shown in Table \ref{table3}. We also take the worst-case ratio of the 10 true distributions, summarized in Table \ref{table4}. From these two tables, we see that the $k$-cluster points become worse as $D$ increases. This agrees with our intuition for one-cluster points in Theorem \ref{no_global}. Also, with more clusters, the ratio is larger since it is easier to fit the data, especially when $D$ is small, and thus a $k$-cluster point may still behave well in this case.

\begin{table}[ht]
\centering
\begin{tabular}{c|ccccccc}
\hline
& $D=4$ & $D=5$ & $D=6$ & $D=8$ & $D = 10$ \\
\hline
$m=2$ & 95.9 & 89.2 & 91.5 & 75.2 & 72.0\\
$m=3$ & 98.7 & 97.9 & 97.1 & 91.9 & 83.4 \\
$m=4$ & 99.5 & 98.1 & 99.1 & 95.1 & 88.4 \\
$m=5$ & 99.9 & 99.8 & 98.9 & 94.9 & 90.3\\
$m=6$ & 100.0 & 99.7 & 99.6 & 95.9 & 92.6\\
\hline
\end{tabular}
\caption{Average ratio of the likelihood (as a percentage) of the $k$-cluster points ($k < m$) that {GD} converges to, to the optimal likelihood. Infinite sample case.}
\label{table3}
\end{table}

\begin{table}[ht]
\centering
\begin{tabular}{c|ccccccc}
\hline
& $D=4$ & $D=5$ & $D=6$ & $D=8$ & $D = 10$ \\
\hline
$m=2$ & 83.3 & 74.3 & 67.7 & 59.1 & 52.7\\
$m=3$ & 91.3 & 91.3 & 82.6 & 73.2 & 60.6 \\
$m=4$ & 90.8 & 83.1 & 92.0 & 83.3 & 71.5 \\
$m=5$ & 96.2 & 94.4 & 94.5 & 80.1 & 82.3\\
$m=6$ & 98.6 & 98.4 & 97.1 & 94.5 & 70.3\\
\hline
\end{tabular}
\caption{Worst-case ratio of the likelihood (as a percentage) of the $k$-cluster points ($k < m$) that {GD} converges to, to the optimal likelihood. Infinite sample case.}
\label{table4}
\end{table}

\section{CONCLUSIONS}\label{sec:conclusion}

In this paper, we identified $k$-cluster regions and studied one-cluster regions in mixture models of two components carefully. For mixtures of two Gaussians, when initialized around one-cluster regions, {EM} escapes such regions exponentially faster than {GD}, as also shown in experiments. For mixtures of two Bernoullis, 
we proved that there exist one-cluster local minima that always trap {GD}, but {EM} can escape such minima. For Bernoulli mixtures with a general number of components and any number of features, we showed experimentally that with random initialization, {EM} always converges to an $m$-cluster point where all components are used, but {GD} often converges to a $k$-cluster point where $k < m$. This means that only a subset of the components are employed. 

Our work opens up a new direction of research. It would be interesting to know theoretically if {EM} almost always escapes $k$-cluster regions in mixture models. A conjecture for BMMs is given in Appendix \ref{conj}.

\subsubsection*{Acknowledgements}
Guojun would like to thank Stephen Vavasis and Yaoliang Yu for useful discussions. 

\bibliographystyle{unsrt}

\begin{thebibliography}{10}

\bibitem{newcomb1886generalized}
Simon Newcomb.
\newblock A generalized theory of the combination of observations so as to
  obtain the best result.
\newblock {\em American journal of Mathematics}, pages 343--366, 1886.

\bibitem{Dempster77}
A.~P. Dempster, N.~M. Laird, and D.~B. Rubin.
\newblock Maximum likelihood from incomplete data via the {EM} algorithm.
\newblock {\em Journal of the Royal Statistical Society, Series B},
  39(1):1--38, 1977.

\bibitem{jeffwu}
CF~Jeff Wu.
\newblock On the convergence properties of the {EM} algorithm.
\newblock {\em The Annals of statistics}, pages 95--103, 1983.

\bibitem{bishop}
Christopher~M. Bishop.
\newblock {\em Pattern Recognition and Machine Learning (Information Science
  and Statistics)}.
\newblock Springer-Verlag New York, Inc., Secaucus, NJ, USA, 2006.

\bibitem{Vlassis:2002}
Nikos Vlassis and Aristidis Likas.
\newblock A {G}reedy {EM} {A}lgorithm for {G}aussian {M}ixture {L}earning.
\newblock {\em Neural Process. Lett.}, 15(1):77--87, February 2002.

\bibitem{Dasgupta:1999}
Sanjoy Dasgupta.
\newblock Learning {M}ixtures of {G}aussians.
\newblock In {\em Proceedings of the 40th Annual Symposium on Foundations of
  Computer Science}, FOCS '99, pages 634--, Washington, DC, USA, 1999. IEEE
  Computer Society.

\bibitem{Dasgupta:2000}
Sanjoy Dasgupta and Leonard~J. Schulman.
\newblock A {T}wo-round {V}ariant of {EM} for {G}aussian {M}ixtures.
\newblock In {\em Proceedings of the Sixteenth Conference on Uncertainty in
  Artificial Intelligence}, UAI'00, pages 152--159, San Francisco, CA, USA,
  2000. Morgan Kaufmann Publishers Inc.

\bibitem{Sanjeev:2001}
Sanjeev Arora and Ravi Kannan.
\newblock {L}earning {M}ixtures of {A}rbitrary {G}aussians.
\newblock In {\em Proceedings of the Thirty-third Annual ACM Symposium on
  Theory of Computing}, STOC '01, pages 247--257, New York, NY, USA, 2001. ACM.

\bibitem{Dasgupta:2007}
Sanjoy Dasgupta and Leonard Schulman.
\newblock A {P}robabilistic {A}nalysis of {EM} for {M}ixtures of {S}eparated,
  {S}pherical {G}aussians.
\newblock {\em Journal of Machine Learning Research}, 8:203--226, May 2007.

\bibitem{Vempala:2004}
Santosh Vempala and Grant Wang.
\newblock A {S}pectral {A}lgorithm for {L}earning {M}ixture {M}odels.
\newblock {\em J. Comput. Syst. Sci.}, 68(4):841--860, June 2004.

\bibitem{Belkin:2010}
Mikhail Belkin and Kaushik Sinha.
\newblock Polynomial {L}earning of {D}istribution {F}amilies.
\newblock In {\em Proceedings of the 2010 IEEE 51st Annual Symposium on
  Foundations of Computer Science}, FOCS '10, pages 103--112, Washington, DC,
  USA, 2010. IEEE Computer Society.

\bibitem{Moitra:2010}
Ankur Moitra and Gregory Valiant.
\newblock Settling the {P}olynomial {L}earnability of {M}ixtures of
  {G}aussians.
\newblock In {\em Proceedings of the 2010 IEEE 51st Annual Symposium on
  Foundations of Computer Science}, FOCS '10, pages 93--102, Washington, DC,
  USA, 2010. IEEE Computer Society.

\bibitem{Hsu:2013}
Daniel Hsu and Sham~M. Kakade.
\newblock Learning {M}ixtures of {S}pherical {G}aussians: {M}oment {M}ethods
  and {S}pectral {D}ecompositions.
\newblock In {\em Proceedings of the 4th Conference on Innovations in
  Theoretical Computer Science}, ITCS '13, pages 11--20, New York, NY, USA,
  2013. ACM.

\bibitem{balakrishnan2017statistical}
Sivaraman Balakrishnan, Martin~J Wainwright, Bin Yu, et~al.
\newblock Statistical guarantees for the {EM} algorithm: From population to
  sample-based analysis.
\newblock {\em The Annals of Statistics}, 45(1):77--120, 2017.

\bibitem{xu2016global}
Ji~Xu, Daniel Hsu, and Arian Maleki.
\newblock Global analysis of expectation maximization for mixtures of two
  gaussians.
\newblock In {\em Proceedings of the 30th International Conference on Neural
  Information Processing Systems}, NIPS'16, pages 2684--2692, USA, 2016. Curran
  Associates Inc.

\bibitem{daskalakis2016ten}
Constantinos Daskalakis, Christos Tzamos, and Manolis Zampetakis.
\newblock Ten {S}teps of {EM} {S}uffice for {M}ixtures of {T}wo {G}aussians.
\newblock In {\em Proceedings of the 2017 Conference on Learning Theory},
  volume~65, pages 704--710, 2017.

\bibitem{Jin:2016}
Chi Jin, Yuchen Zhang, Sivaraman Balakrishnan, Martin~J. Wainwright, and
  Michael~I. Jordan.
\newblock Local {M}axima in the {L}ikelihood of {G}aussian {M}ixture {M}odels:
  {S}tructural {R}esults and {A}lgorithmic {C}onsequences.
\newblock In {\em Proceedings of the 30th International Conference on Neural
  Information Processing Systems}, NIPS'16, pages 4123--4131, USA, 2016. Curran
  Associates Inc.

\bibitem{yan2017convergence}
Bowei Yan, Mingzhang Yin, and Purnamrita Sarkar.
\newblock {C}onvergence of {G}radient {EM} on {M}ulti-component {M}ixture of
  {G}aussians.
\newblock In {\em Proceedings of the 31st International Conference on Neural
  Information Processing Systems}, NIPS'17, pages 6959--6969, USA, 2017. Curran
  Associates Inc.

\bibitem{zhao2018statistical}
Ruofei Zhao, Yuanzhi Li, and Yuekai Sun.
\newblock Statistical {C}onvergence of the {EM} {A}lgorithm on {G}aussian
  {M}ixture {M}odels.
\newblock {\em arXiv preprint arXiv:1810.04090}, 2018.

\bibitem{juan2002use}
Alfons Juan and Enrique Vidal.
\newblock On the use of {B}ernoulli mixture models for text classification.
\newblock {\em Pattern Recognition}, 35(12):2705--2710, 2002.

\bibitem{juan2004bernoulli}
Alfons Juan and Enrique Vidal.
\newblock {B}ernoulli mixture models for binary images.
\newblock In {\em Pattern Recognition, 2004. ICPR 2004. Proceedings of the 17th
  International Conference on}, volume~3, pages 367--370. IEEE, 2004.

\bibitem{conop}
Stephen Boyd and Lieven Vandenberghe.
\newblock {\em Convex Optimization}.
\newblock Cambridge University Press, New York, NY, USA, 2004.

\bibitem{dlreview}
Yann LeCun, Yoshua Bengio, and Geoffrey Hinton.
\newblock Deep learning.
\newblock {\em nature}, 521(7553):436, 2015.

\bibitem{pgd}
Eli~M Gafni and Dimitri~P Bertsekas.
\newblock Two-metric projection methods for constrained optimization.
\newblock {\em SIAM Journal on Control and Optimization}, 22(6):936--964, 1984.

\bibitem{shalev2006efficient}
Shai Shalev-Shwartz and Yoram Singer.
\newblock Efficient learning of label ranking by soft projections onto
  polyhedra.
\newblock {\em Journal of Machine Learning Research}, 7(Jul):1567--1599, 2006.

\bibitem{meyer2000matrix}
Carl~D Meyer.
\newblock {\em Matrix analysis and applied linear algebra}, volume~71.
\newblock Siam, 2000.

\bibitem{lee2016gradient}
Jason~D Lee, Max Simchowitz, Michael~I Jordan, and Benjamin Recht.
\newblock Gradient descent only converges to minimizers.
\newblock In {\em Conference on Learning Theory}, pages 1246--1257, 2016.

\end{thebibliography}

\appendix


\section{Proofs for mixtures of Bernoullis}\label{Append_B}

\subsection{Derivation of \eqref{update_EM_gen}}\label{315}
From \eqref{qtildeBMM}, the update of $\mmu_1$ is:
\be
M(\mmu_1)_i = Z_1^{-1}\int \tilde{q}_1(\x) \x d\x = \frac{\mu_{1i}}{\overline{x}_i} Z_1^{-1}F_i,
\en
where $F_i = \pi_1^* \mu_{1i}^* B_{1i} + \pi_2^* \mu_{2i}^* B_{2i}$ and $B_{1i}, B_{2i}$ defined in \eqref{B1iB2i}. So, 
\be\label{appendA: diff}
M({\bm \l})_i - {\l}_i &=& 2S_i^{-1}\mu_i^* (M(\mmu_1)_i - \mu_{1i}) \tr
&=&2S_i^{-1}\mu_i^* \mu_{1i} \overline{x}_i^{-1} Z_1^{-1} (F_i - Z_1 \overline{x}_i) .\tr
\en
Bringing in the definition of $Z_1$ in \eqref{part_func}, we have 
\be
F_i - Z_1 \overline{x}_i &=& \pi_1^* B_{1i} (\mu_{1i}^* - \overline{x}_i (1 + \pi_2^* \l_i)) + \tr
&+& \pi_2^* B_{2i} (\mu_{2i}^* - \overline{x}_i (1 - \pi_1^* \l_i)).
\en
With the definitions of $\overline{x}_i$, $\mu_i^*$ and $\l_i$, we obtain:
\be
F_i - Z_1 \overline{x}_i = 2\pi_1^* \pi_2^* \mu_i^* (1-\overline{x}_i)^{-1}(1-\mu_{1i})(B_{1i}-B_{2i}),\nonumber
\en
which, combined with \eqref{appendA: diff}, yields \eqref{update_EM_gen}.

\subsection{Proof of Theorem \ref{EM_GD_22}}\label{sec: A.2}
In this section, we prove the following theorem:

\renewcommand{\thetheorem}{4.4}
\begin{theorem}
For $m = D=2$, given $\sigma_{12} \neq 0$ and $\xm \in (0, 1)^D$, with EM algorithm, $\pi_1 = \epsilon$, $\mmu_2 = \xm$ and uniform random initialization for $\mmu_1$, ${\bm \l}$ will converge to the positive regions at a linear rate with probability $1$.  Therefore, EM will almost surely escape one-cluster regions.
\end{theorem}

We assume that $\sigma_{12} > 0$ because the $\sigma_{12} < 0$ can be similarly proved by relabeling $x_2 \to 1 - x_2$. The theorem is equivalent to showing that $\b$ converges to the regions where $b_1 b_2 > 0$, due to \eqref{lambda} and Corollary \ref{pos_cor}. We also call these regions as positive regions. It is not hard to derive from \eqref{qtildeBMM} and \eqref{update_EM} that the EM update is:
\be
\label{diff_1} && b_1 \leftarrow b_1 + Z_1^{-1}\sigma \Lambda_1 b_2,\\
\label{diff_2} && b_2 \leftarrow b_2 + Z_1^{-1}\sigma \Lambda_2 b_1,
\en
with $\Lambda_i = \mu_{1i}(1-\mu_{1i})$. 

We first notice some properties of $\sigma_{12}$, as can be easily seen from its definition. For convenience, in the following proof we define $\sigma := \sigma_{12}S_1^{-1} S_2^{-1}$ which we call the normalized covariance.

\begin{lemma}\label{B_triv}
If $\sigma_{12} > 0$, then $\sigma_{12} < \overline{x}_1 (1 - \overline{x}_2) $, $\sigma_{12} < \overline{x}_2 (1 - \overline{x}_1)$.
\end{lemma}
\begin{proof}
Trivial from the definition of $\sigma_{12}$. 
\end{proof}

A direct consequence is:
\begin{corollary}\label{B_square}
If $\sigma_{12} > 0$, then $\sigma^2 S_1 S_2 < 1$.
\end{corollary}

In the following lemma, we show that in a neighborhood of the origin, $\b$ almost always converges to the positive regions.

\begin{lemma}[\textbf{Convergence with small $||\b||$ initialization, two features}]\label{small_b_2}
Assume $\sigma_{12} > 0$. $\exists \delta > 0$ small enough, with a random $\b$ initialized from the $L_1$ ball $||\b||_1 < \delta$ and the update function defined by {EM}, $\b$ converges to the positive regions $\{(b_1, b_2)|b_1 b_2 > 0\}$ at a linear rate. 
\end{lemma}

\begin{proof}
In this case, $\Lambda_i$ and $Z$ are roughly constant, and
\be
\b' = \begin{bmatrix}
b'_1 \\
b'_2
\end{bmatrix} = {\bf A(\b)}\b  = \begin{bmatrix}
1 & \sigma_{12} \Sigma^{-1}_2\\
\sigma_{12} S_1^{-1} & 1
\end{bmatrix} 
\begin{bmatrix}
b_1 \\
b_2
\end{bmatrix}. 
\en
The eigensystem of ${\bf A(\b)}$ is:
\be
\lambda_1 = 1 + \sigma_{12}\sqrt{S_1^{-1} S_2^{-1}}, \, {\bf v}_1 = (\sqrt{\sigma_{12} S_2^{-1}}, \sqrt{\sigma_{12} S_1^{-1}}).\tr
\lambda_2 = 1 - \sigma_{12}\sqrt{S_1^{-1} S_2^{-1}}, \, {\bf v}_2 = (-\sqrt{\sigma_{12} S_2^{-1}}, \sqrt{\sigma_{12} S_1^{-1}}).\nonumber
\en
Applying ${\bf A}(\b)$ for enough number of times, $\b$ will converge to a multiple of ${\bf v}_1$, where $\sigma b_1 b_2 > 0$. 

Let us make the argument above more concrete. Expand $\b$ as:
\be
\b = c_1 {\bf v}_1 + c_2 {\bf v}_2.
\en
Since $c_1 = 0$ is a measure zero set, we have $c_1 \neq 0$ almost everywhere. WLOG, we assume $c_1 > 0$. Denote
\be
{\bf A}(\b) &=& \begin{bmatrix}
1 & \sigma Z^{-1}\Lambda_1 \\
\sigma Z^{-1}\Lambda_1 & 1 
\end{bmatrix}, \\
{\bf A}(0) &=& \begin{bmatrix}
1 & \sigma_{12}S_2^{-1} \\
\sigma_{12}S_1^{-1} & 1 
\end{bmatrix}.
\en

We first prove that each element of ${\bf A}(\b) - {\bf A}(0)$ is bounded. This can be done by noticing:
\be
{\bf A}(\b) - {\bf A}(0) = \begin{bmatrix}
0 & \sigma (Z^{-1}\Lambda_1 - S_1) \\
\sigma (Z^{-1}\Lambda_2 - S_2) & 0 
\end{bmatrix},\nonumber
\en
and that
\be
Z^{-1}\Lambda_i - S_i = Z^{-1}(b_i(1-2\overline{x}_i) - b_i^2 - \sigma b_1 b_2 S_i).\nonumber
\en
So, ${\bf A}(\b) - {\bf A}(0) = O(\delta)$ in $||\b||_1 < \delta$. From this fact, one can show in $||\b||_1 < \delta$,
\be\label{Abb}
{\bf A}(\b)\b \succeq {\bf A}(0)\b - c\delta^2 {\bf v}_1,
\en
with $c$ some constant. WLOG, we assume $\b$ is still in the negative region and thus $||\b||_1$ decreases, so our approximation is still valid. Applying EM for $k$ times and by use of \eqref{Abb}, we know the result is at least (the generalized inequality is defined by the positive cone $\R^2_{++}$, see, e.g., \cite{conop})
\be
&& {\bf A}(0)^k \b - c\delta^2 (\lambda_1^{k-1} + \dots + \lambda_1 + 1){\bf v}_1\tr
&=&(c_1 \lambda_1^k - c\delta^2 \frac{\lambda_1^k - 1}{\lambda_1 - 1}){\bf v}_1 + c_2 \lambda_2^k {\bf v}_2.
\en
In the above analysis we used ${\bf A}(0) {\bf u} \succeq 0$ for ${\bf u} \succeq 0$. For $\delta$ small enough, we know almost surely $c_1 >  {c\delta^2}/({\lambda_1 - 1})$. Therefore, at almost everywhere ${\bf A}(\b)^k \b$ converges to the positive regions at a linear rate. 

The choice of $L_1$ ball is irrelevant, since in finite vector space all $L_p$ norms are equivalent.
\end{proof}

Now, let us show that in the worst case $\b$ shrinks to a neighborhood of the origin. Hence, combined with Lemma \ref{small_b_2}, we finish the proof. First, rewrite \eqref{diff_1} and \eqref{diff_2} as:
\be
b_1 \leftarrow Z^{-1}(b_1 + \sigma S_1 b_2 + \sigma (1-2\overline{x}_1)b_1 b_2),\\
b_2 \leftarrow Z^{-1}(b_2 + \sigma S_2 b_1 + \sigma (1-2\overline{x}_2)b_1 b_2).
\en

The two contours $b'_1 = 0, b'_2 = 0$ are respectively:
\be
&&C_{b'_1 = 0}: b_2 = f(b_1) = -\frac{b_1}{\sigma (1 - 2\overline{x}_1)b_1+\sigma S_1},\tr
\label{B_b1}\\
&&\label{B_b2} C_{b'_2 = 0}: b_2 = g(b_1) =  \frac{-\sigma S_2 b_1}{1 + \sigma (1-2\overline{x}_2)b_1}.
\en
$f, g$ are both linear fractional functions of $b_1$, an example of which is depicted in Figure \ref{contours}. The derivatives are:
\be
&&f'(b_1) = -\frac{S_1}{\sigma ( (1 - 2\overline{x}_1) b_1 + S_1)^2},\tr
&&g'(b_1) = -\frac{\sigma S_2}{(1 + \sigma (1-2\overline{x}_2)b_1)^2},
\en
therefore, $f, g$ are both decreasing if $\sigma > 0$. It follows that $f'(0) = -(\sigma S_1)^{-1}$ and     $g'(0) = -\sigma S_2$. From Corollary \ref{B_square}, 
\be
\frac{|g'(0)|}{|f'(0)|} < 1.
\en
\begin{figure}
    \centering
    \includegraphics[width=8cm]{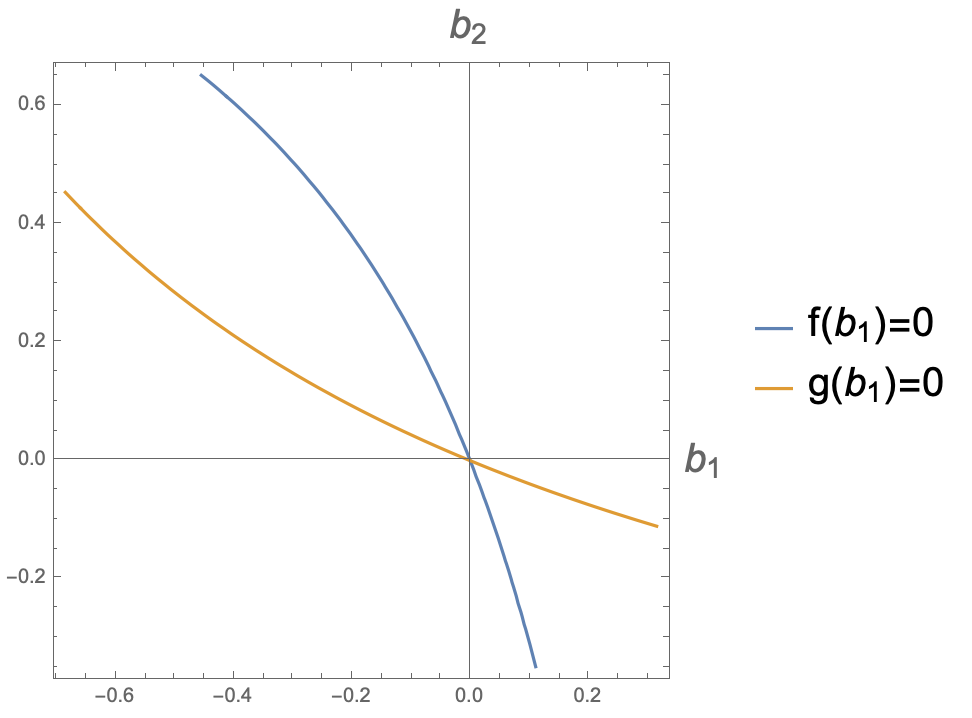}
    \caption{Contours $f(b_1) = 0$ and $g(b_1) = 0$. }
    \label{contours}
\end{figure}
Now, let us look at the secant lines crossing the origin and  $f(-\overline{x}_1)$, $g(-\overline{x}_1)$ separately. From \eqref{B_b1} and \eqref{B_b2},
\be
&& f(-\overline{x}_1) = \frac{1}{\sigma \overline{x}_1}, \\
&& g(-\overline{x}_1) = \frac{\sigma S_2 \overline{x}_1}{1 - \sigma \overline{x}_1 (1-2\overline{x}_2)},
\en
and one can obtain $f(1 - \overline{x}_1)$ and $g(1 - \overline{x}_1)$ similarly. So, $-b_1 /b_2$ is bounded in the following region: 
\be
 \{(b_1, b_2)|b_1 b_2 < 0, f(b_1) g(b_1) < 0\},
\en
and the bound is given by the slopes of the tangent lines at $(0, 0)$ and the secant lines. 

Another important point to notice is that $f(b_1)$ and $g(b_1)$ intersect exactly once. Which can be proved from Lemma \ref{B_triv}, \eqref{B_b1} and \eqref{B_b2}:
\begin{lemma}\label{region_part}
Assume $\sigma > 0$, $f(b_1) = g(b_1)$ has exactly one solution $b_1 = 0$ in the feasible region $b_1 \in [-\overline{x}_1, 1 - \overline{x}_1]$.
\end{lemma}
\begin{proof}
The solution $b_1 = 0$ is obvious. For $b_1 \neq 0$, $f(b_1) = g(b_1)$ is equivalent to:
\be\label{cross_point}
\sigma^2 S_2 \left( (1-2\overline{x}_1)b_1 + S_1\right)- \sigma (1-2\overline{x}_2)b_1 = 1.
\en
We will show that the left hand side is always less than one. This is a linear function, so we only need to show it at both end points. At $b_1 = -\overline{x}_1$, the left hand side can be simplified as:
\be
\sigma\overline{x}_1 (1-\overline{x}_2) + \sigma \overline{x}_1 \overline{x}_2 \left( 
\sigma\overline{x}_1 (1-\overline{x}_2) - 1 \right) < 1,\nonumber
\en
where we used Lemma \ref{B_triv}. Similarly, at $b_1 = 1-\overline{x}_1$, the left hand side of \eqref{cross_point} is:
\be
\sigma\overline{x}_2 (1-\overline{x}_1) + \sigma (1-\overline{x}_1)(1 - \overline{x}_2) \left( 
\sigma\overline{x}_2 (1-\overline{x}_1) - 1 \right) < 1.\nonumber
\en
\end{proof}

From Lemma \ref{region_part}, the feasible region of $\b$ is divided into four parts by $f(b_1)$ and $g(b_1)$. If $f(b_1)g(b_1) > 0$, then EM update goes to the positive region. Otherwise, $f(b_1)g(b_1) < 0$. In this region, neither $b_1$ nor $b_2$ changes the sign. Because $-b_1/b_2$ is bounded, from \eqref{diff_1} and \eqref{diff_2}, $||\b||^{(t+1)}_1 \leq q||\b||^{(t)}_1$ with $0 < q < 1$ being a constant. So, in the worst case, $\b$ will converge to the neighborhood of the origin at a linear rate, and then shift to the positive regions at a linear rate, according to Lemma \ref{small_b_2}. This lemma can be used because the EM update is not singular: it does not map a measure nonzero set to a measure zero set. Hence, after finitely many steps, at the neighborhood of the origin, the random distribution at the beginning is still random.

Figure \ref{examples} is an example of the trajectory. 
\begin{figure}
    \centering
    \includegraphics[width=8cm]{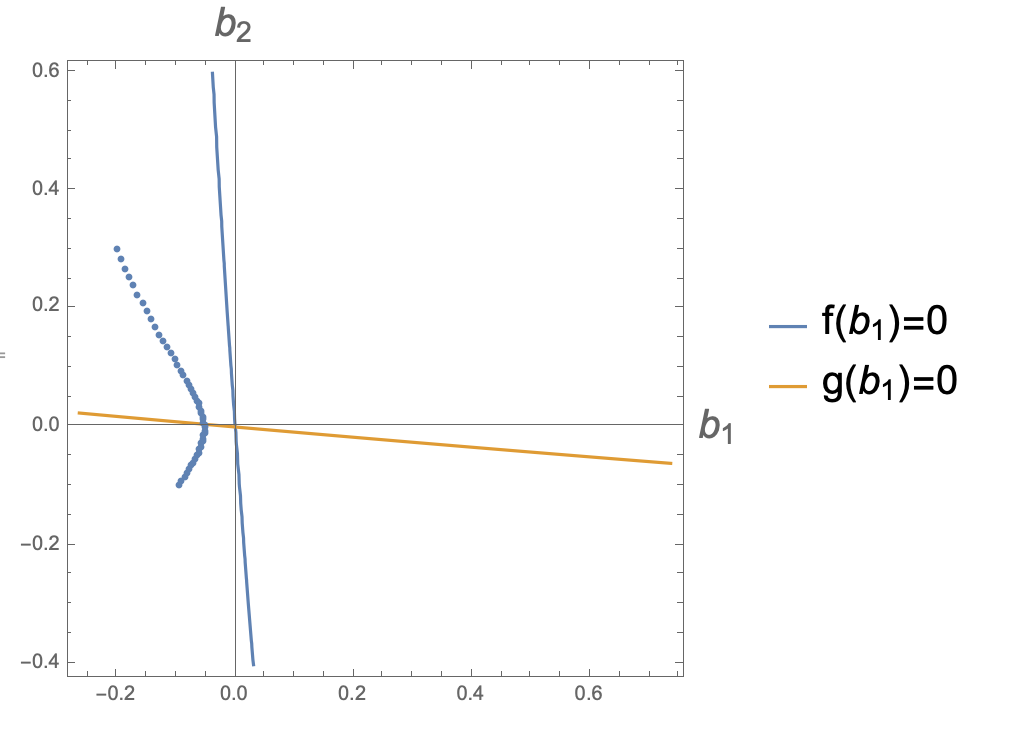}
    \caption{An example of the trajectory. $\b$ moves from the second orthant to a positive region, by shrinking its norm and rotating. }
    \label{examples}
\end{figure}

\section{A General Conjecture}\label{conj}

In this appendix, we propose a general conjecture for mixtures of Bernoullis: 
\begin{conjecture}\label{Conj2}
For any number of clusters and a general number of features, with random initialization around $k$-cluster regions, EM will almost always converge to an $m$-cluster point.
\end{conjecture}

Besides empirical evidence, we can show theoretical guarantees at $m=2$. In this case, the problem reduces to showing the convergence to positive regions proposed in Section \ref{pos_reg}. The convergence to the positive regions is observed empirically for mixtures of two Bernoullis, which always happens with random initialization. 

The first result shows that if $||{\bm \l}||$ is small, ${\bm \l}$ will converge to the positive regions:
\begin{proposition}[\textbf{Convergence with small $||{\bm \l}||$ initialization, general}]\label{small_b_gen}
For mixtures of two Bernoullis, $\exists \delta > 0$ small enough, with a random ${\bm \l}$ initialized from the $L_1$ ball $||{\bm \l}||_1 < \delta$ and the update function defined by {EM}, ${\bm \l}$ will converge to the positive regions.
\end{proposition}

\begin{proof}
Around $||{\bm \l}|| \sim 0$, $Z_1 \sim 1$. Expanding \eqref{update_EM_gen} to the linear order, we find that the update of ${\bm \l}$ can be linearized as:
\be
M({\bm \l}) = {\bf A} {\bm \l},
\en
where $A_{ii} = 1$ and $A_{ij} = (2\mu_i^*)^2 \pi_1^* \pi_2^* S_i^{-1} Z_1^{-1} > 0$ for $i \neq j$. 
 Hence, ${\bf A}$ is an irreducible matrix. After enough iterations, ${\bm \l}$ will converge to the linear span of the largest eigenvector of ${\bf A}$. 

From the Perron-Frobenius theorem \cite{meyer2000matrix}, ${\bf A}$ has a unique largest real eigenvalue, and  ${\rm eig}_{\rm max}({\bf A})  \geq \min_i \sum_j {A_{ij}} > 1$. Also, the maximal eigenvector of ${\bf A}$, ${\bf v}_{\rm max}$, is a multiple of an all positive vector. Therefore, we can prove the proposition in a similar fashion as the proof of  Lemma \ref{small_b_2}.
\end{proof}

This proposition also tells us that EM has an effect of rotating ${\bm \lambda}$ to the positive regions. It is interesting to observe that such unstable fixed point ${\bm \l} = 0$ is analogous to the strict saddle points studied in \cite{lee2016gradient}. It might be possible to use stable manifold theorem to prove our conjecture at $m=2$.

Another special case is when the $\max_i\lambda_i$ increases or $\min_i\lambda_i$ decreases. For a given ${\bm \lambda}$, we can order the components. WLOG, we assume $\l_1 \leq \l_2 \leq \dots \l_i < 0 = \l_{i+1} = \dots = \l_{j} < \l_{j+1}\leq \dots \leq \l_D$. We can show the following proposition:
\begin{proposition}\label{order_b_gen}
For mixtures of two Bernoullis, assume $\l_1 \leq \l_2 \leq \dots \l_i < 0 = \l_{i+1} = \dots = \l_{j} < \l_{j+1}\leq  \dots \leq \l_D$, if $M({\bm \l})_1 < \lambda_1$ or $M({\bm \l})_D > \lambda_D$, then ${\bm \l}$ will eventually converge to the positive regions. Otherwise, we have $M({\bm \l})_1 \geq \lambda_1$ and $M({\bm \l})_D \leq \lambda_D$.
\end{proposition}

\begin{proof}
From the definitions of $B_{1i}$ and $B_{2i}$, \eqref{B1iB2i}, we have
\be
B_{11} - B_{21} \geq \dots \geq B_{1D} - B_{2D}.
\en
$M_1({\bm \l}) < \lambda_1$, from \eqref{update_EM_gen}, tells us that $B_{11} - B_{21} < 0$, and thus every $\l_i$ decreases. As each $\l_i$ decreases, $B_{1i} - B_{2i}$ will get smaller as well. Therefore, all $\l_i$'s decrease at least as a linear function. Since the feasible region is bounded, ${\bm \l}$ will converge to $-\R^D_{++}$ eventually. 

Similarly, if $M_D({\bm \l}) > \lambda_D$, we know that ${\bm \l}$ will converge to $\R^D_{++}$ eventually.

Otherwise, we must have $B_{11} - B_{21} \geq 0$ and $B_{1D} - B_{2D} \leq 0$, yielding $M_1({\bm \l}) \geq \lambda_1$ and $M_D({\bm \l}) \leq \lambda_D$.
\end{proof}
The two patterns $M({\bm \l})_1 < \lambda_1$ and $M({\bm \l})_D > \lambda_D$ have been observed in experiments very frequently, while the case with $M({\bm \l})_1 \geq \lambda_1$ and $M({\bm \l})_D \leq \lambda_D$ needs some further understanding.

\section{Proof for mixtures of two Gaussians with fixed equal covariances}\label{2-gmm-gen}

In this appendix, we generalize our analysis for mixtures of two Gaussians with identity covariances to a slightly more general setting of mixtures of two Gaussians, given that the two covariance matrices are equal and fixed. 
The result is almost the same as Theorem \ref{balance_EM}, except that we replace the normal inner product in Euclidean space to a new inner product defined as $\langle a, b\rangle_{\Sigma} = a^T \S^{-1} b$. 

\begin{theorem}\label{balance_EM_c}
Consider a mixture of two Gaussians with $\pi_1^*\in (0, 1)$, the same covariance matrices and true distribution
\be
p^*(x) = \pi^*_1 \N(\x|\mmu^*, \S) + \pi^*_2 \N(\x|-\mmu^*, \S)
\en
When $\pi_1 = \epsilon$ is initialized to be small enough, $\mmu_2 = \xm$ and $\b^T \S^{-1} \mmu^* \neq 0$, EM increases $\pi_1$ exponentially fast.
\end{theorem}
\begin{proof}
It suffices to show that $Z_1 > 1$ and $Z_1$ grows if initially $\b^T  \S^{-1} \mmu^* \neq 0$. To calculate $Z_1$, we first compute $\tilde{q}_1(\x) = p^*(\x)\gamma_1(\x)$:
\be\label{tilde_q_2_c}
\tilde{q}_1(\x) &=& \pi_1^* e^{\b^T \S^{-1} (\mmu^* - \mmu_2)}\N(\x| \mmu^* +\b, \S) \tr
&+& \pi_2^* e^{-\b^T \S^{-1} (\mmu^* + \mmu_2)} \N(\x| -\mmu^* + \b, \S).\tr
\en
This equation shows that $\tilde{q}_1(\x)$ corresponds to an un-normalized mixture of Gaussians with their means shifted by $b$ and their mixing coefficients rescaled in comparison to $p^*(\x)$.
If $\b = {\bf 0}$, then $\tilde{q}_1(\x) = p^*(\x)$. So, $\b$ describes how different $\tilde{q}_1(\x)$ deviates from $p^*(\x)$.

The partition function can be computed by integrating out $\tilde{q}_1(\x)$:
\be
Z_1 =  \pi_1^* e^{\b^T\S^{-1}(\mmu^* - \mmu_2)} + \pi_2^* e^{-\b^T \S^{-1} (\mmu^* + \mmu_2)}.
\en
In fact, $\mmu_2 = \overline{\x} = (\pi_1^* - \pi_2^*)\mmu^*$ which can be derived from $p^*(\x)$. So, $Z_1$ becomes:
\be\label{Z_q_c}
Z_1 =  \pi_1^* e^{2\pi_2^* \b^T  \S^{-1}\mmu^*} + \pi_2^* e^{-2\pi_1^* \b^T \S^{-1} \mmu^*}.
\en
Using the fact $e^x \geq 1 + x$ and that equality holds iff $x = 0$, we can show that when $\b^T \S^{-1} \mmu^* \neq 0$, $Z_1 > 1$. 

Now, let us show that $Z_1$ increases. It suffices to prove that $|\b^T \S^{-1} \mmu^*|$ increases. From \eqref{tilde_q_2_c}, we have the update equation for $\mmu_1$: $\mmu_1 \leftarrow (\pi'_1 - \pi'_2)\mmu^* + \b$,
 with
\be\label{pi1prime_c}
\pi'_1 = \pi_1^* Z_1^{-1}  e^{2\pi_2^* \b^T \S^{-1} \mmu^*}, \pi'_2 = \pi_2^* Z_1^{-1}  e^{-2\pi_1^* \b^T \S^{-1} \mmu^*}.\nonumber
\en
If $\b^T \S^{-1} \mmu^* > 0$, then $\pi'_1 > \pi_1^*$ and $\pi'_2 < \pi_2^*$. So, $\mmu_1^{(t+1)} = \mmu_1^{(t)} + \delta \mmu^*$ with $\delta = \pi'_1-\pi_1^* + \pi_2^* - \pi'_2> 0$, and
\be\label{bt+1_c}
(\b^{(t+1)})^T \S^{-1} \mmu^* &=& (\b^{(t)})^T  \S^{-1} \mmu^* + \delta \mmu^T \S^{-1} \mmu \tr
&>& (\b^{(t)})^T \S^{-1} \mmu^*,
\en
where we use the superscript $(t+1)$ to denote the updated values and $(t)$ to denote the old values. 

Similarly, we can prove that $\b^T \S^{-1}\mmu^* $ will decrease if $\b^T \S^{-1}\mmu^* < 0$. Hence, from \eqref{Z_q_c}, $Z_1$ increases under EM if initially $\b^T \S^{-1}\mmu^* \neq 0$.
\end{proof}

\end{document}